\documentclass[letterpaper]{article}
\usepackage[preprint]{aaai2027}
\usepackage[hyphens]{url}
\usepackage{graphicx}
\DeclareUrlCommand\code{\urlstyle{tt}}
\usepackage{natbib}
\usepackage{caption}
\usepackage{amsmath,amssymb,amsthm}
\usepackage{booktabs}

\newtheorem{definition}{Definition}
\newtheorem{proposition}{Proposition}

\newcommand{\imacs}{\textsc{IMACS}}
\newcommand{\eg}{e.g.,\ }
\newcommand{\ie}{i.e.,\ }

\title{Toward an Organizational Science of Multi-Agent LLM Systems: \\
Decoupling Who, How, and Which Algorithm}

\author{
Huan Chen\textsuperscript{\rm 1},
Xiang Song\textsuperscript{\rm 1},
Jian Jin\textsuperscript{\rm 1},
Pan Ren\textsuperscript{\rm 1},
Liang-Jie Zhang\textsuperscript{\rm 2}
}
\affiliations{
\textsuperscript{\rm 1}Shunfeng Technology Co., Ltd., China\\
\textsuperscript{\rm 2}Shenzhen University, China
}

\begin{document}
\maketitle

\begin{abstract}
Multi-agent frameworks built on large language models (LLMs) routinely entangle three
logically distinct concerns: \emph{who} is on the team (organization), \emph{how} members
align (coordination), and \emph{which} algorithm fuses their work (collaboration protocol).
\imacs{} (Intelligent Multi-Agent Collaboration System) separates the three into orthogonal,
independently swappable layers. Classic organizational theory (Belbin roles, Mintzberg
coordination, RACI accountability) becomes executable, validated configuration, and the
framework places six published collaboration algorithms behind a common interface while
exposing roles, coordination, and accountability as independently configurable factors. We
use this separation to conduct controlled comparisons in which organizational assignments
vary while the collaboration protocol is held fixed. It also turns protocol choice into a
variable that can be \emph{learned}: Adaptive Org Routing, a contextual-bandit meta-protocol,
selects a protocol per task under an explicit quality--cost trade-off, outperforms every
fixed protocol in a controlled study, and trains online on real benchmark and LLM-judge
rewards. The ablations expose a mechanism. Accountability placement changes outcomes exactly
when the protocol routes the deliverable through the accountable agent, and the winning
placement flips across model families, so organizational design cannot be hard-coded; it must
be re-validated, or learned, for each model binding.
\end{abstract}

\section{Introduction}

Large language models (LLMs) have moved from single-shot prompting to \emph{multi-agent}
collaboration, in which multiple model instances critique and build on one another's
reasoning. A growing
toolbox of collaboration algorithms now
exists---Mixture-of-Agents~\citep{wang2024moa}, Multi-Agent Debate~\citep{du2023debate},
Self-Consistency~\citep{wang2023selfconsistency}, LLM-Blender~\citep{jiang2023llmblender},
Plan-and-Execute~\citep{hong2024metagpt}, Reflexion~\citep{shinn2023reflexion}---alongside agent
frameworks such as AutoGen, CAMEL, and MetaGPT.

\paragraph{The entanglement problem.}
Despite this progress, most frameworks \emph{tangle} three logically independent concerns
into one implementation:
\begin{enumerate}
  \item \textbf{Organization (WHO):} which agents exist, what roles they play, which models
  back them, and who is accountable.
  \item \textbf{Coordination (HOW):} the mechanism by which members align, whether informal
  mutual adjustment, a supervisor issuing orders, or standardized procedures and outputs.
  \item \textbf{Collaboration protocol (the ALGORITHM):} the procedure that fuses individual
  contributions into a result, such as debate, voting, or planning.
\end{enumerate}
When these are fused, swapping the aggregation algorithm means rewriting the team. The cost is
threefold: teams cannot be reconfigured without touching
algorithms; decades of organizational theory~\citep{belbin2010,mintzberg1979} cannot be applied as
a first-class design surface; and collaboration algorithms cannot be compared on an equal
organizational footing, since each ships its own implicit team.

\paragraph{Our thesis: decouple the three concerns.}
We argue these concerns are \emph{orthogonal} and model them as independently swappable layers.
\imacs{} (Intelligent Multi-Agent Collaboration System) realizes this: organizations are declared
in YAML and grounded in Belbin's nine roles~\citep{belbin2010}; coordination is one of Mintzberg's
five mechanisms~\citep{mintzberg1979}; and the collaboration protocol is a pluggable object behind
a single interface. RACI is a cross-cutting tag with a validated at-most-one-Accountable
constraint, and any layer can be replaced without touching the others
(Proposition~\ref{prop:ortho}). To show the organization layer is culturally pluggable, \imacs{}
ships three presets: \texttt{belbin} (default), \texttt{adhocracy}, and
\texttt{three-departments} (the Tang-dynasty ministry system in Belbin terms).

\paragraph{Why decouple?}
Once the three factors are independently swappable, organizational variables (roles,
coordination, RACI) can be held fixed or ablated while the protocol stays constant; to our
knowledge, no prior multi-agent LLM framework supports such controlled ablation of \emph{who}
and \emph{how} independent of \emph{which algorithm}. Protocol choice also becomes a free
variable. We exploit both consequences in this paper.

\paragraph{Adaptive Org Routing.}
With protocol choice freed, we cast it as an online decision problem. \textbf{Adaptive Org
Routing} featurizes each task into an interpretable 7-dimensional vector and uses a LinUCB
contextual bandit~\citep{li2010linucb} to select, per task, the protocol that maximizes a
cost-adjusted reward $r=\text{quality}-\lambda\cdot\text{cost}$, updating online via rank-1
Sherman--Morrison updates. Over a task stream it traces a quality/cost Pareto frontier,
spending on expensive protocols only when they earn their cost.

\paragraph{Contributions.}
\begin{itemize}
  \item \textbf{A decoupled architecture and formalization (our core contribution).} We separate
  organization, coordination, and collaboration into orthogonal layers, formalize an organization
  as $O=(\mathcal{R},\mu,\kappa,\rho)$ and a protocol as
  $\pi:(\text{Task},\text{Team},\text{Blackboard})\!\to\!\text{Result}$, and prove an orthogonality
  property on which the controlled ablations and the router below both rest.
  \item \textbf{A systematic study of the organization dimension.} Holding the protocol fixed
  and varying the organization, controlled RACI ablations replicated three times show
  accountability matters
  \emph{exactly when} the protocol routes the deliverable through the accountable agent, and
  only by changing which model does the work; a natural cross-dataset conjecture proves
  \emph{unstable across repeats and model bindings}. We also stress-test a catalog of
  evidence-backed (organization, protocol) recipes on five real benchmarks.
  \item \textbf{Adaptive Org Routing}, an online organization-aware meta-protocol that selects
  a protocol per task via a contextual bandit over task features, evaluated on a controlled
  oracle ($82\%$ routing accuracy, beating all six fixed baselines) and trained on real
  objective and LLM-judge rewards.
  \item \textbf{A modular, reproducible framework}, released with this paper: a unified
  pluggable implementation of the six protocols above on AgentScope~2.0, with gateway key
  isolation, a transparent blackboard runtime, and seeded benchmarks.
\end{itemize}

A note on scope: the framework is the primary contribution and the router a method instance
built on it. The router's headline number comes from a controlled simulated oracle; the
real-benchmark, real-reward, and RACI results are small-$N$, reported as direction and
mechanism (see Discussion).

\section{Related Work}

\paragraph{Multi-agent LLM frameworks.}
Prior frameworks couple roles, coordination, and aggregation~\citep{guo2024survey}:
conversational systems like AutoGen, CAMEL, and AgentVerse fix the dialogue loop as the
protocol~\citep{wu2023autogen,li2023camel,chen2024agentverse}, while MetaGPT and ChatDev bind
roles to a fixed
software-company pipeline~\citep{hong2024metagpt,qian2024chatdev}. A complementary line
optimizes the communication topology itself---GPTSwarm learns over agent
graphs~\citep{zhuge2024gptswarm}, DyLAN prunes dynamic agent networks~\citep{liu2024dylan}---%
but \emph{within} one algorithmic family, without declarative roles or accountability.
Purpose-built evaluations have followed: MultiAgentBench scores collaboration across
coordination topologies with milestone KPIs~\citep{zhu2025multiagentbench}; AgentsNet poses
formally verifiable coordination problems over agent
networks~\citep{groetschla2025agentsnet}. \imacs{}
instead treats organization, coordination, and collaboration as orthogonal layers over a
generic runtime (we build on AgentScope~2.0~\citep{gao2024agentscope}), so a team and an
algorithm combine freely.

\paragraph{Collaboration algorithms.}
We treat published collaboration algorithms as interchangeable \emph{protocols}:
Mixture-of-Agents~\citep{wang2024moa} (proposer/aggregator layers), Multi-Agent
Debate~\citep{du2023debate} (critique-revise rounds), Self-Consistency~\citep{wang2023selfconsistency}
(sample-and-vote), LLM-Blender~\citep{jiang2023llmblender} (rank-and-fuse), Plan-and-Execute
in the Plan-and-Solve / MetaGPT line~\citep{wang2023planandsolve,hong2024metagpt}
(decompose-then-execute), and
Reflexion~\citep{shinn2023reflexion} (verbal self-feedback). \imacs{} ships each behind a common
interface, so they can be compared on the same team and selected automatically per task.

\paragraph{Organizational theory.}
Belbin~\citep{belbin2010} identifies nine team roles (thought-, people-, action-oriented) a
balanced team covers; Mintzberg~\citep{mintzberg1979} identifies five coordination mechanisms and
forms such as the adhocracy; RACI assigns Responsible/Accountable/Consulted/Informed tags,
classically one Accountable per deliverable. While agent ``personas'' are common, \imacs{} makes
these theories a \emph{first-class, declarative, validated} design surface, with cultural
pluggability via the \emph{three-departments} preset (the Tang ministry system;
\citealt{fei2026institutions} also cast historical institutions as agent topologies, but as
fixed architectures, not a swappable layer).

\paragraph{Bandits for model and protocol selection.}
LinUCB~\citep{li2010linucb} is the canonical linear-payoff contextual
bandit~\citep{auer2002,chu2011contextual} with
$\tilde{O}(\sqrt{T})$ regret. Adjacent selection problems use the same machinery, routing
\emph{queries to models}~\citep{chen2024frugalgpt,shnitzer2024routing,ong2025routellm,poon2026multillm} or
across agent \emph{communication} standards (A2A, ACP,
ANP)~\citep{du2026protocolbench}. \imacs{} routes at a different layer, to our knowledge unaddressed:
its arms are entire \emph{collaboration algorithms} over a declared organization, its context
an interpretable 7-dimensional task featurization, and the problem is well-posed only because
organization and protocol are decoupled (Proposition~\ref{prop:ortho}).

\section{The \imacs{} Framework}
\label{sec:framework}

\begin{figure}[t]
\centering
\includegraphics[width=0.95\columnwidth]{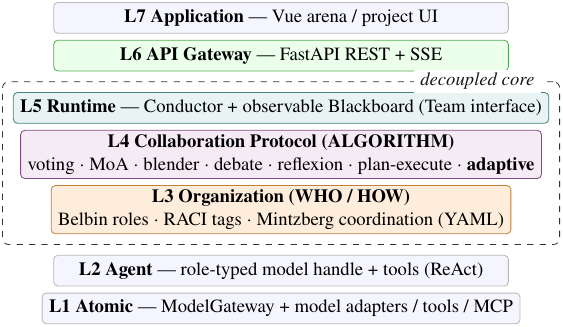}
\caption{The seven-layer \imacs{} stack. Inside the dashed region, the organization
(WHO/HOW) and the collaboration protocol (ALGORITHM) meet only through the runtime's Team and
Blackboard interfaces, so any organization composes with any protocol
(Proposition~\ref{prop:ortho}). The adaptive router is itself an L4 protocol that selects
among the others.}
\label{fig:arch}
\end{figure}

\imacs{} is organized as a stack of seven layers (Figure~\ref{fig:arch}). Lower layers provide
atomic capabilities; upper layers compose them. The three middle layers (organization,
collaboration, and the runtime that hosts them) are where the decoupling lives.

\subsection{Layers}
\textbf{L1 Atomic capabilities.} Model adapters, tools, MCP servers, and the provider-facing
\emph{ModelGateway}. The ModelGateway owns provider credentials and resolves opaque model
handles, keeping credentials below the agent boundary. Models are restricted to an
author-defined experimental registry: callable through available provider APIs, compatible with
the common AgentScope chat interface, and spanning distinct provider families. It is an
implementation allowlist, not an external certification or quality endorsement. An offline
deterministic stand-in keeps tests and demos running without credentials or network access.

\textbf{L2 Agent.} A role-typed wrapper binding a model handle, a system prompt, and tools;
agents act via a ReAct-style reason-act loop~\citep{yao2023react} when tools are involved.

\textbf{L3 Organization (WHO/HOW).} A declarative team: Belbin-typed roles, each bound to a
registered model and a RACI tag, plus a Mintzberg coordination mechanism. Loaded from YAML
presets and formalized below (Definition~\ref{def:org}).

\textbf{L4 Collaboration protocol (ALGORITHM).} The pluggable fusion algorithm behind the
single interface of Definition~\ref{def:protocol}. Ships voting, MoA, blender, debate,
reflexion, plan-execute, and the \emph{adaptive} meta-protocol (Adaptive Org Routing).

\textbf{L5 Runtime.} A \emph{Conductor} instantiates a Team from an organization and drives a
protocol over a shared \emph{Blackboard} on which every intermediate message is posted, a
transparent room with human-in-the-loop and SSE hooks.

\textbf{L6 API Gateway.} A FastAPI boundary exposes the L5 Conductor through REST and streams
Blackboard events through SSE. It delegates model resolution to L1's ModelGateway and does not
expose provider credentials to applications.

\textbf{L7 Application.} A Vue application provides two canonical demos: an \emph{arena} (models
compete per query) and a \emph{project} (a task decomposed across roles).

\subsection{Formalization}
\label{sec:formal}

\begin{definition}[Organization]
\label{def:org}
An organization is a tuple $O=(\mathcal{R},\mu,\kappa,\rho)$ where
\begin{itemize}
  \item $\mathcal{R}=\{r_1,\dots,r_n\}$ is a set of \emph{roles}, each $r_i\in\mathcal{B}$ drawn from
  Belbin's nine team roles $\mathcal{B}$;
  \item $\mu:\mathcal{R}\to\mathcal{M}$ is a \emph{model binding} from roles to the configured
  model registry $\mathcal{M}$ (the gateway guarantees $\mathrm{im}(\mu)\subseteq\mathcal{M}$);
  \item $\kappa\in\mathcal{K}$ is a single \emph{coordination mechanism} from Mintzberg's five
  mechanisms $\mathcal{K}$;
  \item $\rho:\mathcal{R}\to\{\mathsf{R},\mathsf{A},\mathsf{C},\mathsf{I}\}$ is the RACI
  \emph{assignment}, subject to the single-owner constraint
  $\bigl|\{r\in\mathcal{R}: \rho(r)=\mathsf{A}\}\bigr|\le 1$.
\end{itemize}
\end{definition}

Here $\mathsf{R}$, $\mathsf{A}$, $\mathsf{C}$, and $\mathsf{I}$ denote Responsible,
Accountable, Consulted, and Informed; Greek symbols name organization components rather than
reusing these RACI initials. The single-owner constraint is validated at load time. A \emph{Team} $T$ is the instantiation
of $O$: a set of agents
$\{a_i\}$ where $a_i$ wraps model $\mu(r_i)$ under role $r_i$, sharing a coordination context
derived from $\kappa$.

\begin{definition}[Collaboration protocol]
\label{def:protocol}
A collaboration protocol is a function
\[
  \pi:\ (\mathit{Task},\ \mathit{Team},\ \mathit{Blackboard})\ \longrightarrow\ \mathit{Result},
\]
that, given a task, an instantiated team, and a shared blackboard, executes a fusion algorithm
and returns a result, while posting every intermediate utterance to the blackboard. The set of
protocols $\Pi$ is extensible: a new algorithm is a new $\pi\in\Pi$ registered by name.
\end{definition}

A protocol is deliberately \emph{agnostic} to which roles populate the team and to $\kappa$; it
sees a pool of role-typed agents, never a fixed cast. Conversely, an organization carries no
commitment to any $\pi$.

\begin{proposition}[Orthogonality of the three layers]
\label{prop:ortho}
Let $\mathcal{O}$ be the set of valid organizations and $\Pi$ the protocols. The system's behavior is a well-defined map
$\mathrm{run}:\mathcal{O}\times\Pi\times\mathit{Task}\to\mathit{Result}$ on the full product
$\mathcal{O}\times\Pi$: any organization $O\in\mathcal{O}$ composes with any protocol
$\pi\in\Pi$. Hence replacing one component leaves the interfaces of the others invariant:
for fixed $\pi$, varying $O$ (roles, models, coordination $\kappa$, or RACI assignment $\rho$) requires
no change to $\pi$; and for fixed $O$, varying $\pi$ requires no change to $O$.
\end{proposition}

\begin{proof}[Proof sketch]
$\pi$ depends on $O$ only through the \emph{Team} interface (a pool of role-typed agents) and the
\emph{Blackboard} interface, neither of which exposes specific roles, model identities,
$\kappa$, or $\rho$. $O$ depends on $\pi$ through no interface at all; it is a static declaration
consumed by the runtime, not by $\pi$. Therefore the two are free coordinates of the design
space, and $\mathrm{run}$ is total on $\mathcal{O}\times\Pi$. The coordination mechanism $\kappa$ is
likewise a property of $O$ surfaced to agents as context, independent of $\pi$. $\square$
\end{proof}

Proposition~\ref{prop:ortho} lets the three presets (\texttt{belbin}, \texttt{adhocracy},
\texttt{three-departments}) and the protocol library be reused combinatorially, and it makes
the routing problem of Adaptive Org Routing well-posed: the router chooses a point in $\Pi$
while $O$ is held fixed.

\section{Adaptive Org Routing}
\label{sec:routing}

Decoupling protocols from organizations (Proposition~\ref{prop:ortho}) makes protocol choice a
free variable. The right choice is task-dependent: cheap voting suffices for factual lookups,
debate helps reasoning, and plan-and-execute suits decomposable work, but each step up costs
more model calls. \emph{Adaptive Org Routing} learns this mapping online.

\subsection{Task featurization}
Routing must be cheap relative to the protocols it selects, so featurization uses no model call.
A goal string is mapped to an interpretable vector $x\in\mathbb{R}^{7}$ whose coordinates are an
intercept $1$ and six features---\emph{length}, \emph{\#questions}, \emph{\#subtasks},
\emph{reasoning}, \emph{factoid}, \emph{has\_code}---each bounded to $[0,1]$ by normalized
counts (token length, \texttt{?} count, decomposition/reasoning/factoid keyword densities, and a
code indicator). These features are the \emph{context} that lets the policy generalize across
tasks of the same type.

\subsection{Bandit formulation}
We cast per-task protocol selection as a contextual bandit. The arms are the six base protocols
$\mathcal{A}=\{\text{voting},\text{moa},\text{blender},\text{debate},\text{reflexion},\text{plan-execute}\}$.
On task $t$ with feature vector $x_t$, the router plays arm $\pi_t\in\mathcal{A}$, runs that
protocol over the (fixed) organization, observes a quality estimate $q_t\in[0,1]$ from an
evaluator agent, and forms a \emph{cost-adjusted reward}
\begin{equation}
  r_t \;=\; q_t \;-\; \lambda\cdot \tilde{c}(\pi_t),
  \qquad
  \tilde{c}(\pi)=\frac{c(\pi)}{\max_{\pi'}c(\pi')},
  \label{eq:reward}
\end{equation}
where $c(\pi)$ is the protocol's relative model-call multiplicity (voting $1.0$, MoA/blender
$1.5$, reflexion $2.0$, plan-execute $2.5$, debate $3.0$) and $\lambda\ge 0$ tunes cost aversion.

We use LinUCB~\citep{li2010linucb}. Each arm $a$ maintains $A_a\in\mathbb{R}^{d\times d}$
(initialized to $I_d$, $d{=}7$) and $b_a\in\mathbb{R}^{d}$. Its payoff estimate is
$\hat\theta_a=A_a^{-1}b_a$ and the router selects
\begin{equation}
  \pi_t \;=\; \arg\max_{a\in\mathcal{A}}\ \Bigl(\hat\theta_a^{\top}x_t
  \;+\;\alpha\sqrt{x_t^{\top}A_a^{-1}x_t}\,\Bigr),
  \label{eq:ucb}
\end{equation}
the upper-confidence bound on cost-adjusted utility, with exploration coefficient $\alpha$.
After observing $r_t$ it updates $A_{\pi_t}\!\mathrel{+}=x_tx_t^{\top}$,
$b_{\pi_t}\!\mathrel{+}=r_t x_t$. We maintain $A_a^{-1}$ directly via the
Sherman--Morrison rank-1 update,
\begin{equation}
  A^{-1} \leftarrow A^{-1} - \frac{(A^{-1}x)(A^{-1}x)^{\top}}{1+x^{\top}A^{-1}x},
  \label{eq:sm}
\end{equation}
giving $O(d^2)$ per-step learning with no matrix inversion. The full per-task meta-protocol is
given as pseudocode in the supplement.

\subsection{Persistence and online observability}
The router's entire learned state is the per-arm pair $(A_a^{-1}, b_a)$ plus a pull
count, a few small dense matrices ($d{=}7$) that serialize compactly to JSON, so a policy can
be checkpointed and restored to resume learning rather than cold-start; restoration validates
the arm set and feature dimension. Because learning is fully online, rolling routing accuracy and
cumulative regret stream per round (Figure~\ref{fig:curves}): train once, persist, keep
improving.

\subsection{Regret and ``org-awareness''}
Equation~\eqref{eq:ucb} is LinUCB on a $d$-dimensional context with bounded rewards, which under
the standard linear-payoff assumption enjoys $\tilde{O}(\sqrt{dT})$ cumulative
regret~\citep{li2010linucb}; our setting satisfies the assumptions by construction (full
statement and sketch in the supplement), so per-round regret against the cost-adjusted optimum
vanishes as $\tilde{O}(\sqrt{d/T})$. The router holds $O$ fixed and chooses over $\Pi$, but by
Proposition~\ref{prop:ortho} the arm set extends to (protocol, coordination) pairs or per-role
model bindings with the same machinery (only $|\mathcal{A}|$ grows); we evaluate the
protocol-selection instance and leave the joint space to future work.

\section{Scenario Recommendations}
\label{sec:recipes}

Proposition~\ref{prop:ortho} guarantees that any organization composes with any protocol,
but not \emph{which} combination to pick. This section closes that gap with a small catalog
of evidence-backed (organization, protocol) recipes (Table~\ref{tab:recipes}), each naming a
public benchmark so the recommendation is falsifiable, shipped as a one-call API
and an auto-recommender.

\subsection{From evidence to recipes}
The recipes follow established results: self-consistency voting for
verifiable arithmetic~\citep{wang2023selfconsistency}; multi-agent debate for open-ended
reasoning and factuality~\citep{du2023debate}; Mixture-of-Agents~\citep{wang2024moa} and
ranking-fusion~\citep{jiang2023llmblender} for open-ended generation; Reflexion for code with
executable feedback~\citep{shinn2023reflexion}; and role-specialized planning for complex,
decomposable tasks~\citep{hong2024metagpt}.

\begin{table}[t]
\centering
\small
\begin{tabular}{@{}lll@{}}
\toprule
\textbf{Scenario} & \textbf{Org} & \textbf{Protocol} \\
\midrule
Factual lookup        & belbin            & voting        \\
Math reasoning        & belbin            & voting        \\
Open reasoning        & belbin            & debate        \\
Code generation       & belbin            & reflexion     \\
Open-ended gen.       & adhocracy         & moa           \\
Candidate fusion      & belbin            & blender       \\
Complex decomp.       & three-departments & plan-execute  \\
\bottomrule
\end{tabular}
\caption{Evidence-backed scenario recipes. Each pairing is the default \imacs{}
recommendation for its scenario, tested against a public benchmark; evidence cited in the text.}
\label{tab:recipes}
\end{table}

Collaboration is not universally beneficial: heavy protocols can over-think simple tasks and
coordination failures can reduce quality~\citep{cemri2025mast}. Recipes therefore encode
structural role division rather than persona strings, which do not reliably improve objective
accuracy~\citep{zheng2023persona}, and serve as priors for the adaptive router.

\subsection{From asserted to validated recommendations}
\label{sec:recipes:realeval}
We test these priors on five public benchmarks: GSM8K~\citep{cobbe2021gsm8k}
exact-match, HumanEval~\citep{chen2021humaneval} pass@1, HotpotQA token-F1, and
AgentsNet~\citep{groetschla2025agentsnet} formal validity use task-intrinsic, model-free
scoring; open-ended AlpacaEval uses a disclosed LLM-judge win-rate. Runs also log latency and
answer length, isolate item-level transport failures, and record data source and generation
mode. Every table cell is generated from its evaluation record, so Table~\ref{tab:realeval}
is a directly auditable quality/cost frontier rather than a manually transcribed summary.

\begin{table*}[t]
\centering
\small
\setlength{\tabcolsep}{4pt}
\begin{tabular}{@{}lccccc@{}}
\toprule
 & \textbf{GSM8K} & \textbf{HumanEval} & \textbf{HotpotQA} & \textbf{AlpacaEval} & \textbf{AgentsNet} \\
\textbf{Protocol} & exact-match & pass@1 & token-F1 & win-rate & validity \\
\midrule
Voting          & 0.60 & \textbf{1.00} & 0.33 & 0.90 & \textbf{1.00} \\
Debate          & \textbf{0.90} & \textbf{1.00} & 0.36 & 1.00 & \textbf{1.00} \\
Reflexion       & 0.50 & \textbf{1.00} & 0.21 & -- & \textbf{1.00} \\
Plan-Execute    & \textbf{0.90} & \textbf{1.00} & \textbf{0.54} & -- & \textbf{1.00} \\
MoA             & -- & -- & -- & 1.00 & -- \\
Blender         & -- & -- & -- & 1.00 & -- \\
\midrule
Best solo & 0.90 & 1.00 & 0.53 & 1.00 & 1.00 \\
\bottomrule
\end{tabular}
\caption{Real-benchmark evaluation, $N{=}10$ items/cell, real models (belbin org). Each
column is scored by its objective metric (exact-match / pass@1 / token-F1 / formal
validity) or, for the inherently open-ended AlpacaEval, an LLM-judge win-rate vs.\ a
reference baseline (judge = GLM-5.2; a model judge carries length/style biases). Best
protocol per objective column in bold. At $N{=}10$, reported AlpacaEval scores are near the
ceiling and AgentsNet scores equal $1.00$, so we read cost there. ``--'': not run. Bottom row: the strongest \emph{single} model per dataset
answering alone (mean over the three solo runs in the organization-dimension study;
single run for AlpacaEval and AgentsNet).}
\label{tab:realeval}
\end{table*}

\paragraph{Findings.}
At $N{=}10$/cell, HotpotQA supports Plan-and-Execute ($0.54$, also lowest latency), while
GSM8K favors Plan-and-Execute and Debate ($0.90$) over the Voting recipe ($0.60$).
At $N{=}10$, all reported HumanEval and AgentsNet cells equal $1.00$, leaving cost as the
useful axis: Plan-and-Execute is respectively $4.2\times$ and $2.3\times$ faster than Debate.
Reported AlpacaEval scores are near the ceiling ($0.90$--$1.00$); MoA reaches the top observed
win-rate at the lowest latency ($127$ s versus $154$ s
for Blender and $230$ s for Debate). No protocol exceeds the strongest solo model on quality,
and that model changes by dataset. Thus task type and model binding jointly determine the
winner: recipes are falsifiable priors, not guarantees, motivating adaptive selection.
These small-$N$ runs are directional rather than leaderboard claims. Objective metrics are
used wherever possible; the open-ended judge, its biases, and the synthetic/real boundary are
explicitly disclosed. AgentsNet is evaluated as a whole-graph deliverable because its checker
is intrinsic to the output; framework-coupled interaction metrics are excluded (supplement).

\subsection{The organization dimension}
\label{sec:orggrid}
The findings above vary the protocol while holding the organization fixed. We now cross the
two. Table~\ref{tab:orggrid}(a) evaluates a partial $3\times2$ organization--protocol grid
on GSM8K: no organization is uniformly best, the protocol preference varies with the
organization (\emph{belbin} and \emph{three-departments} favor debate, $0.60/0.80$ and
$0.50/0.90$, while \emph{adhocracy} is even at $0.70/0.70$), and a fixed protocol
moves by up to $0.20$ exact-match across organizations. This is direct evidence for
Proposition~\ref{prop:ortho} in practice: the organization is a real, independent lever whose best choice
interacts with the protocol.

\paragraph{Does accountability placement matter? A controlled RACI ablation.}
Because the three presets differ on many axes at once, a raw grid cannot isolate \emph{RACI}, so
we add a controlled ablation (Table~\ref{tab:orggrid}(b)): three variants of the \emph{same}
belbin team (identical roles, models, and coordination) differing only in the Accountable tag
(Chair, Critic, or none). Voting and debate never read the tag, so for them the three variants
are functionally identical, which lets this grid double as a noise calibration. Over
three independent repeats, the variant \emph{means} agree within $0.04$ (voting
$0.57$--$0.60$; debate $0.73$--$0.77$) while the \emph{single-run} cross-variant spread
reaches $0.17\pm0.05$: single-run differences of that size are generation noise, the yardstick
for the Blender effect below. With three repeats (all $\pm$ are population std over repeats), we
report sign consistency and paired effect sizes rather than significance tests.
We report this near-null result as is: \emph{where} accountability sits barely moves
objective accuracy, consistent with evidence that role/persona tags alone do not lift objective
performance~\citep{zheng2023persona} and with the framing above that a recipe's value is
structural.

\paragraph{When RACI \emph{does} bite: a mechanistic prediction, tested.}
The null on voting/debate does not mean accountability is inert; those protocols simply never
read the Accountable tag. Reading the control flow predicts when RACI matters: LLM-Blender hands the
final fusion to the team's accountable agent (first agent when none is tagged), so the
\emph{fuser is the accountable agent}; moving Accountable from the Chair (\texttt{glm-5.2})
to the Critic (\texttt{deepseek}) reroutes the final write to a different model. MoA instead
aggregates at its coordinator role and consults the Accountable tag only as a fallback,
so its aggregator stays the Chair regardless of RACI. This yields a
falsifiable prediction, \emph{RACI should move Blender but not MoA}, and a natural follow-up
conjecture, that Blender's winning variant should use the stronger fuser model. We test both
across three independent repeats with a solo-model baseline
(Table~\ref{tab:orggrid}(c) and Table~\ref{tab:raci:crossds}). The prediction holds in the
repeated data: on HotpotQA, Blender's paired Critic$-$Chair gap is positive in \emph{every}
repeat ($+0.08\pm0.02$ in the fuser grid; $0.34\pm0.06$ vs.\ $0.48\pm0.05$ cross-dataset),
while the MoA control, whose aggregator is identical in all three variants, shows an unsigned
spread of comparable magnitude ($0.13$ per repeat) whose \emph{ordering shuffles} across
repeats. The effect is stable in sign; the control varies without one. GSM8K moves the same
way when it moves ($0.70\pm0.00$ vs.\ $0.77\pm0.09$) and all HumanEval cells equal $1.0$ in this sample. The follow-up
conjecture, by contrast, has \emph{no stable direction}: across the repeats the winning fuser
matched the stronger \emph{solo} model in $1$, $3$, and $1$ of 3 datasets, and shifts again under
a different model binding (supplement). The claim that survives is
precise: \textbf{accountability placement affects outcomes exactly when the protocol routes
the deliverable through the accountable agent (a replicated, sign-consistent effect), it
operates by changing which model does the work, and its direction is not predictable a
priori}. The optimal placement must therefore be found empirically, an argument for an
adaptive router over (organization, protocol) rather than a static rule.

\begin{table}[t]
\centering
\small
\begin{tabular}{@{}lcc@{}}
\toprule
\textbf{Organization} & \textbf{Voting} & \textbf{Debate} \\
\midrule
\multicolumn{3}{@{}l}{\emph{(a) org $\times$ protocol --- presets}} \\
belbin              & 0.60 & \textbf{0.80} \\
adhocracy           & 0.70 & 0.70 \\
three-departments   & 0.50 & \textbf{0.90} \\
\midrule
\multicolumn{3}{@{}l}{\emph{(b) RACI ablation (3 repeats)}} \\
Chair   & 0.60$\pm$0.00 & 0.77$\pm$0.05 \\
Critic           & 0.57$\pm$0.05 & 0.73$\pm$0.05 \\
none     & 0.60$\pm$0.16 & 0.73$\pm$0.12 \\
\midrule
\multicolumn{3}{@{}l}{\emph{(c) HotpotQA fuser ablation (3 repeats)}} \\
 & \textbf{MoA} & \textbf{Blender} \\
Chair   & 0.38$\pm$0.03 & 0.34$\pm$0.04 \\
Critic           & 0.32$\pm$0.07 & 0.42$\pm$0.02 \\
none     & 0.30$\pm$0.06 & 0.25$\pm$0.06 \\
\bottomrule
\end{tabular}
\caption{Organization dimension (exact-match / token-F1, $N{=}10$/cell, real models;
(b--c) mean$\pm$std over three repeats).
\emph{(a)} GSM8K: no organization is uniformly best and the protocol preference varies with
the organization --- a real, independent lever that interacts with the protocol.
\emph{(b)} Voting and debate never read the Accountable tag, so these variants are
functionally identical: means agree within $0.04$, single-run spread reaches $0.17\pm0.05$
(the measured $N{=}10$ noise floor).
\emph{(c)} Blender's fuser \emph{is} the accountable agent: its paired Critic$-$Chair gap is
positive in $3/3$ repeats ($+0.08\pm0.02$), while the MoA control (identical aggregator)
shuffles across repeats. The gap keeps its sign; the control does not.}
\label{tab:orggrid}
\end{table}

\begin{table}[t]
\centering
\small
\setlength{\tabcolsep}{2.5pt}
\begin{tabular}{@{}lcccc@{}}
\toprule
\textbf{Dataset} & \textbf{Blender} & \textbf{Blender} & \textbf{Solo} & \textbf{Solo} \\
 & \textbf{Chair} & \textbf{Critic} & \textbf{glm} & \textbf{deepseek} \\
 & glm fuser & deepseek fuser & & \\
\midrule
GSM8K & 0.70$\pm$0.00 & 0.77$\pm$0.09 & 0.70$\pm$0.00 & 0.90$\pm$0.08 \\
HumanEval & 1.00$\pm$0.00 & 1.00$\pm$0.00 & 1.00$\pm$0.00 & 1.00$\pm$0.00 \\
HotpotQA & 0.34$\pm$0.06 & 0.48$\pm$0.05 & 0.53$\pm$0.04 & 0.50$\pm$0.01 \\
\bottomrule
\end{tabular}
\caption{RACI $\times$ Blender cross-dataset attribution, mean$\pm$std over 3 independent repeats ($N{=}10$/cell, real models). Moving Accountable from Chair (\texttt{glm-5.2} fuser) to Critic (\texttt{deepseek} fuser) reroutes who writes the final answer. The winning fuser matched the stronger \emph{solo} model in 1, 3, 1 of 3 datasets per repeat: no stable direction, so the optimal placement must be found empirically.}
\label{tab:raci:crossds}
\end{table}

\section{Experiments}
\label{sec:experiments}

We study whether Adaptive Org Routing \emph{learns} to send each task to its cost-adjusted
optimal protocol, and whether the learned policy beats every fixed protocol. Because the goal is
to measure routing/regret behavior against \emph{known} ground truth, not to estimate absolute
LLM quality, we evaluate against a controlled simulated oracle (motivated in
the Discussion).

\subsection{Setup}
\textbf{Task stream.} A mixed suite of factual, reasoning, complex-decomposition, code, and
open-ended tasks (\eg ``What is the capital of France?''). Each round draws a task uniformly
at random; tasks carry a type tag used only to
define the oracle, never seen by the router.

\textbf{Simulated quality oracle.} Each (task-type, protocol) pair has a latent mean quality in
$[0,1]$ plus bounded Gaussian noise ($\sigma{=}0.03$). The latent profile encodes established
literature trends: factual tasks peak at \emph{voting}; reasoning peaks at
\emph{debate}~\citep{du2023debate}; code peaks at \emph{reflexion}~\citep{shinn2023reflexion};
complex/decomposable tasks peak at \emph{plan-execute}~\citep{hong2024metagpt}; open-ended
generation peaks at \emph{MoA}~\citep{wang2024moa}. After cost
adjustment (Eq.~\ref{eq:reward}) each type has one clearly separated optimal arm, giving a
learnable signal and a well-defined regret target.

\textbf{Router and baselines.} LinUCB with $d{=}7$, $\alpha{=}0.6$. We compare against the six
fixed-protocol baselines (always play one arm), an \emph{oracle} upper bound (always play the
cost-adjusted optimal arm, no noise), and a \emph{context-free} bandit ablation (LinUCB fed a
constant intercept-only feature, \ie no task context). Reward is cost-adjusted
(Eq.~\ref{eq:reward}); unless noted $\lambda{=}0.3$ and $T{=}400$. Routing accuracy is the
fraction of the last 50 tasks routed to the cost-adjusted optimal arm. All runs are
seeded and reproducible (supplement).

\textbf{Model versions.} The controlled study above uses no LLM calls. All real-model runs use
this registry and pin exact versions: DeepSeek-V4
(\code{deepseek-v4-pro}, \code{deepseek-v4-flash}), Qwen3.7 (\code{qwen3.7-plus}),
GLM-5.2 (\code{glm-5.2}), and MiniMax-M3
(\code{MiniMax-M3}); the AlpacaEval judge is \code{glm-5.2}.

\paragraph{Main result: the router approaches the oracle.}
Table~\ref{tab:main} reports mean cost-adjusted reward at $T{=}400$, $\lambda{=}0.3$. The router
attains $0.708$, beating \emph{all six} fixed-protocol baselines (best fixed, MoA, $0.669$)
and recovering $45\%$ of the best-fixed$\to$oracle headroom (oracle $0.754$); final routing
accuracy is $82\%$, with sublinear regret and an accuracy plateau (Figure~\ref{fig:curves}).

\begin{table}[t]
\centering
\small
\setlength{\tabcolsep}{3pt}
\begin{tabular}{lcc}
\toprule
Method & Mean reward & $\Delta$ vs.\ router \\
\midrule
\textbf{Adaptive Org Routing (LinUCB)} & \textbf{0.708} & --- \\
\midrule
Oracle (upper bound) & 0.754 & $+0.046$ \\
\midrule
MoA (best fixed) & 0.669 & $-0.038$ \\
Blender & 0.652 & $-0.056$ \\
Voting & 0.629 & $-0.079$ \\
Reflexion & 0.621 & $-0.086$ \\
Plan-Execute & 0.550 & $-0.157$ \\
Debate & 0.519 & $-0.189$ \\
\midrule
Context-free bandit (ablation) & 0.628 & $-0.080$ \\
\bottomrule
\end{tabular}
\caption{Mean cost-adjusted reward over $T{=}400$ mixed tasks ($\lambda{=}0.3$). The router
beats all six fixed baselines and approaches the oracle; the context-free ablation collapses
to always-voting, below the best fixed arm.}
\label{tab:main}
\end{table}

\begin{figure}[t]
\centering
\includegraphics[width=0.98\columnwidth]{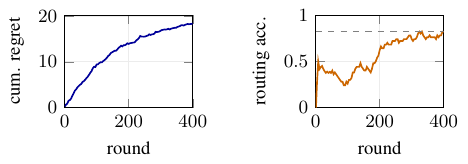}
\caption{Online learning ($\lambda{=}0.3$, $T{=}400$). \emph{Left:} cumulative regret vs.\ the
cost-adjusted oracle grows sublinearly and flattens as the policy converges (consistent with
the $\tilde{O}(\sqrt{T})$ LinUCB bound). \emph{Right:} last-50 routing accuracy climbs toward
its plateau (dashed: final $82\%$).}
\label{fig:curves}
\end{figure}

The learned arm counts, in turn, show what the router learned: cheap \emph{MoA} (142/400) and
\emph{voting} (111) carry the open-ended and factual mass, \emph{reflexion} (73) fires on code
tasks, and the expensive arms (\emph{plan-execute} 29, \emph{blender} 23, \emph{debate} 22)
are played only where their edge survives the cost penalty.

\paragraph{Ablation: task context.}
The context-free bandit (constant feature) collapses to mean reward $0.628$, the level of
always-voting ($0.629$) and \emph{below} the best fixed arm (MoA, $0.669$). Without features it
can only learn one global arm, so the gain comes from context, not from the bandit machinery.

\paragraph{Ablation: cost-aversion $\lambda$.}
The $\lambda$-sweep ($\lambda\in\{0.1,0.3,0.5\}$) shows the router beats best-fixed at every
setting: accuracy rises as cost pressure separates the cost-adjusted optima ($70\%$, $82\%$,
$84\%$) while absolute reward falls with the heavier penalty.

\subsection{From synthetic oracle to real rewards}
\label{sec:realrouting}
We now run the \emph{same} router on \emph{real} rewards: each round draws an item from a mixed
real-benchmark stream, the chosen protocol runs with real models, and the objective score
(exact-match / pass@1 / token-F1) forms $r=\text{score}-\lambda\,\tilde c(\pi)$. With no oracle
for real quality we report realized reward and score, not a regret curve. Over an 18-round mixed stream
($\lambda{=}0.3$, belbin) the router
trains end-to-end with no errors: mean task score $0.68$, reward $0.54$
(HumanEval $1.00$, GSM8K $0.83$, HotpotQA $0.21$). We read this as a feasibility result
($N{=}18$, with the arm distribution concentrated on one cost-efficient protocol): the central
mechanism now runs unchanged on genuine outcomes.

\paragraph{Judge-based rewards integrate on the same scale.}
Mixing AlpacaEval into the stream (reward an LLM-judge win-rate against a reference baseline), a
20-round four-dataset run trains with no errors (mean task score $0.86$,
reward $0.76$):
judge-scored items land in $[0,1]$ on the same scale as exact-match/F1, so the
bandit needs no special-casing. As a small-$N$ artifact the policy again concentrates on one
cheap arm (here voting): a subjective, judge-derived reward drops into the same online loop.

\section{Discussion and Limitations}
\label{sec:discussion}

\paragraph{The synthetic oracle is a controlled stand-in.}
The regret/accuracy study uses a \emph{simulated} oracle calibrated to literature
trends~\citep{du2023debate,shinn2023reflexion,hong2024metagpt}: learning under known ground
truth with exact regret; real-benchmark and real-reward results supply absolute quality.

\paragraph{Linearity, arms, and scale.}
LinUCB assumes payoffs approximately linear in the 7 features; kernelized or neural bandits
could capture nonlinear structure at higher cost. Extending arms to (protocol, coordination)
or per-role bindings is well-defined under Proposition~\ref{prop:ortho} but enlarges
exploration; our real runs are small-$N$ feasibility runs, not tuned policies.

\paragraph{Cost and statistical scope.}
The router uses model-call multiplicity rather than provider-specific billing, so deployment
costs require local recalibration. Real-model cells contain only $N{=}10$ items and the RACI
study only three repeats; we therefore report latency, dispersion, sign consistency, and
paired effects instead of leaderboard or significance claims. The process-global policy is
suited to sequential learning; concurrent services should isolate or synchronize updates.

\paragraph{Benchmark and deployment scope.}
The evaluated tasks score final deliverables; they do not cover long-horizon tool use, human
approval loops, or failures caused by changing external state. AgentsNet is posed as a
whole-graph team deliverable rather than through its native partial-visibility harness, which
preserves task-intrinsic scoring but narrows the coordination claim. Online runs also depend on
provider availability and a unified proxy, so their latency is environment-specific even when
quality metrics are reproducible. Finally, the current feature map is intentionally lightweight
and English-centric. Richer semantic features, persistent policies across deployments, and
joint routing over organization and per-role model bindings remain open evaluations rather than
claims established here.

\paragraph{Scope of orthogonality.}
Proposition~\ref{prop:ortho} guarantees \emph{interfaces} compose, not that every combination
is equally \emph{effective}; the router exists to discover effective combinations empirically.

\paragraph{Model-binding dependence.}
Recipe winners and the direction of the RACI effect flip across model bindings (supplement),
so organizational conclusions need not transfer across model families. Re-validation must
therefore be cheap, or learned online; this is the framework's central argument.

\section{Conclusion}

Most multi-agent LLM frameworks hard-wire \emph{who} collaborates into \emph{how} results are
fused. \imacs{} separates these concerns: organizational theory becomes executable
configuration, controlled ablations expose mechanisms, and protocol choice becomes learnable.
Adaptive Org Routing outperforms every fixed protocol in the controlled study and trains online
on real rewards. Roles, coordination, and accountability can thus be optimized alongside fusion
algorithms, enabling inspectable comparisons across organizations, protocols, and model bindings.

{\small
\bibliography{references}
}

\end{document}